\documentclass[12pt]{article}

\RequirePackage{kpfonts}
\RequirePackage[sf,bf,small,raggedright]{titlesec}
\usepackage{dsfont}
\usepackage{comment}
\usepackage{textcomp}
\usepackage{amssymb}
\usepackage{bbm}
\usepackage{fancyhdr}
\usepackage{amsmath}
\usepackage{amsbsy,amsthm}
\usepackage{amscd}
\usepackage{latexsym}
\usepackage{graphicx}   
\usepackage{pdfsync}
\usepackage{blkarray}
\usepackage{multirow}
\usepackage{hyperref}
\usepackage{xcolor}
\usepackage{enumitem}

\definecolor{csim}{rgb}{0.3,0.7,0}

\usepackage{tcolorbox}

\usepackage{tikz}

\usepackage{subcaption}
\usepackage{algorithmic}
\usepackage[justification=centering]{caption}

\usepackage[longnamesfirst,authoryear,sort&compress]{natbib}
\def\newblock{\ }%
\bibpunct[, ]{[}{]}{,}{n}{}{,}%

\usepackage{float}
\usepackage{pgf}
\usetikzlibrary{arrows}

\newcommand{\R}{\mathbb{R}}

\newcommand{\N}{\mathbb{N}}

\newcommand{\E}{\mathbb{E}}

\newtheorem{lemma}{Lemma}
\newtheorem{theorem}{Theorem}
\newtheorem{corollary}{Corollary}
\newtheorem{proposition}{Proposition}

\allowdisplaybreaks

\numberwithin{equation}{section}

\def\IND{\mathbbm{1}}

\newcommand{\EXP}{\mathbb{E}}
\newcommand{\PROB}{\mathbb{P}}

\newcommand{\var}{\mathrm{Var}}

\begin{document}

\title{Broadcasting in random recursive {\sc dags}
    \thanks{Simon Briend acknowledges the support of Région Ile de France.
G\'abor Lugosi acknowledges the support of Ayudas Fundación BBVA a
Proyectos de Investigación Científica 2021 and
the
Spanish Ministry of Economy and Competitiveness grant PID2022-138268NB-I00, financed by MCIN/AEI/10.13039/501100011033,
FSE+MTM2015-67304-P, and FEDER, EU.}
\author{
    Simon Briend \\
  Department of Mathematics, \\
  Unidistance Suisse, \\
  3900, Brigue, Switzerland \\
  \and
  Luc Devroye  \\
  School of Computer Science\\
  McGill University \\
  Montreal, Canada
\and  
G\'abor Lugosi \\
Department of Economics and Business, \\
Pompeu  Fabra University, Barcelona, Spain \\
ICREA, Pg. Lluís Companys 23, 08010 Barcelona, Spain \\
Barcelona School of Economics
}
}

\maketitle

\begin{abstract}

  A uniform $k$-{\sc dag} generalizes the uniform random recursive
  tree by picking $k$ parents uniformly at random from the existing
  nodes. It starts with $k$ ''roots''. Each of the $k$ roots is
  assigned a bit. These bits are propagated by a noisy channel. The
  parents' bits are flipped with probability $p$, and a majority vote
  is taken. When all nodes have received their bits, the $k$-{\sc dag}
  is shown without identifying the roots. The goal is to estimate the
  majority bit among the roots. We identify the threshold for $p$ as a
  function of $k$ below which the majority rule among all nodes yields
  an error $c+o(1)$ with $c<1/2$.
Above the threshold the majority rule errs with
  probability $1/2+o(1)$.
\end{abstract}

\section{Introduction}

\vspace{0.5cm}

The interest in network analysis has been growing, in part due to its use in
communication technologies, social network studies, and biology, see
\citet{10.1093/oso/9780198709893.001.0001}. The problem we study here
is the one of broadcasting on random graphs. We study the setting
where a bit propagates with noise and we want to infer the value of
the original bit. The question is not if and how the information
propagates, but if there is a signal propagating on the graph, or only
noise. Variations of this binary classification problem have been
studied. For example, in the root-bit estimation problem, the root of
a tree has a bit $0$ or $1$. The value of this bit propagates from the
root to the leafs, and at each propagation from a vertex to the next
it mutates (flips the bit) with probability $p$. One can try to infer
the root's bit value from observing all the bits of the graph or only
the leaf bits. This question was first formulated in
\citet{10.1214/aoap/1019487349} on general trees, where it was shown
that root bit reconstruction is possible depending upon a condition on
the branching number. More recently, the case of random recursive
trees
(\citet{AdDeLuVe22,desmarais2021broadcasting})
has been studied. Other variations of these problems on trees include
looking at asymmetric flip probabilities (\citet{10.1214/10-AOP584}),
non-binary vertex values (\citet{mossel2001reconstruction}) and
robustness to perturbation (\citet{janson2004robust}). We refer the
reader to \citet{mossel2004survey} for a survey of reconstruction
problems on trees. Many problems are described by more general graphs
rather than trees. The original broadcasting question has been studied
on deterministic graphs (\citet{harutyunyan2020new}) and Harary graphs
(\citet{7023754}, for example). We are interested in the problem of
noisy propagation in the spirit of the root-bit reconstruction
(\citet{10.1214/aoap/1019487349}), but on a class of random graphs
that we call $k$-{\sc dag} (for directed acyclic graph). A similar
problem -- for a different class of random {\sc dag}s -- has been
studied in \citet{Makur_2020}.
In a related probelm, \citet{antunovic2016coexistence} studied the
case of the preferential attachment model, where initial nodes have a
color and the color of the new nodes is a function of the colors of their neighbors.

Since we track the proportion of zero bits in our graph, we cast the
process as an urn model. A similar reformulation was already done in
\citet{AdDeLuVe22} to study majority
voting properties of broadcasting on random recursive trees. The
proportion of zero bits and the bit assignment procedure can be viewed
as random processes with reinforcement. A review of results can be
found in \citet{Pemantle_2007} and is extensively used, alongside
results of non-convergence found in
\citet{pemantle1990nonconvergence}. As in
\citet{AdDeLuVe22}, we make ample
use of the properties of P\'olya urns
(\citet{wei1979generalized,janson2004functional,knape2014polya}). Variations
of the P\'olya urn model that are useful for our analysis include an
increase of the number of colors over time
(\citet{https://doi.org/10.48550/arxiv.2204.03470}), the selection of
multiple balls in each draw (\citet{MR3666709}), and randomization in
the color of the new ball
(\citet{10.1214/19-ECP226,zhang2022convergence}). We
note, in particular, the multi-ball draw with a linear randomized
replacement rule of \citet{CRIMALDI2022270}. In the present paper, we
consider multi-ball draws, but with non-linear randomized replacement.

The paper is organized as follows. 
After introducing the mathematical model in Section \ref{sec:model}, in Section \ref{sec:results}
we present the main result of the paper (Theorem \ref{thm:MainThm}) 
that shows that there are three different regimes of the value of the
mutation probability that characterize the asymptotic behavior of the majority rule.
In Section \ref{sec:DifRegime} we discuss the three regimes of $p$. 
In Section \ref{sec:convergence} we establish convergence properties of the global proportion 
of both bit values assigned to vertices and in Section \ref{sec:majority} we finish the proof 
of Theorem \ref{thm:MainThm} by studying the probability of error in all three regimes. 
Finally, in Section \ref{sec:lower} we establish a lower bound for the probability of error 
that holds uniformly for all mutation probabilities. We conclude the paper by discussing 
avenues for further research.

\subsection{The model}
\label{sec:model}

We start by describing the evolution of the uniform random recursive $k$-{\sc dag} and the assigned bit values
that we represent by two colors; red and blue.

Let us fix an odd integer $k>0$. The growth process is initiated at time $k$. At time $k$, the graph consists of
$k$ isolated vertices.
A fraction $R_k$ are red and a fraction
$B_k=1-R_k$ are blue. We set $R_1=\cdots=R_k$ and
$B_1=\cdots=B_k$. The network is grown recursively by adding a new
colored vertex and at most $k$ edges at each time step. At time $n$,
a new vertex $n$ connects to a sample of $k$ vertices chosen
uniformly at random with replacement among the $n-1$ previous
vertices. (Possible multiple edges are collapsed into one so that the graph remains simple.)
The color of vertex $n$ is determined by the following
randomized rule:

\begin{itemize}
\item the colors of the $k$ selected parents are observed;
\item each of these is independently flipped with probability $p$ (if a parent is selected more than once, its color is flipped independently for each selection);
\item the color of vertex $n$ is chosen according to the majority vote
  of the flipped parent colors
 (i.e., there are exactly $k$ votes).
\end{itemize}

\begin{figure}[H]
\begin{center}
\includegraphics[scale=0.8]{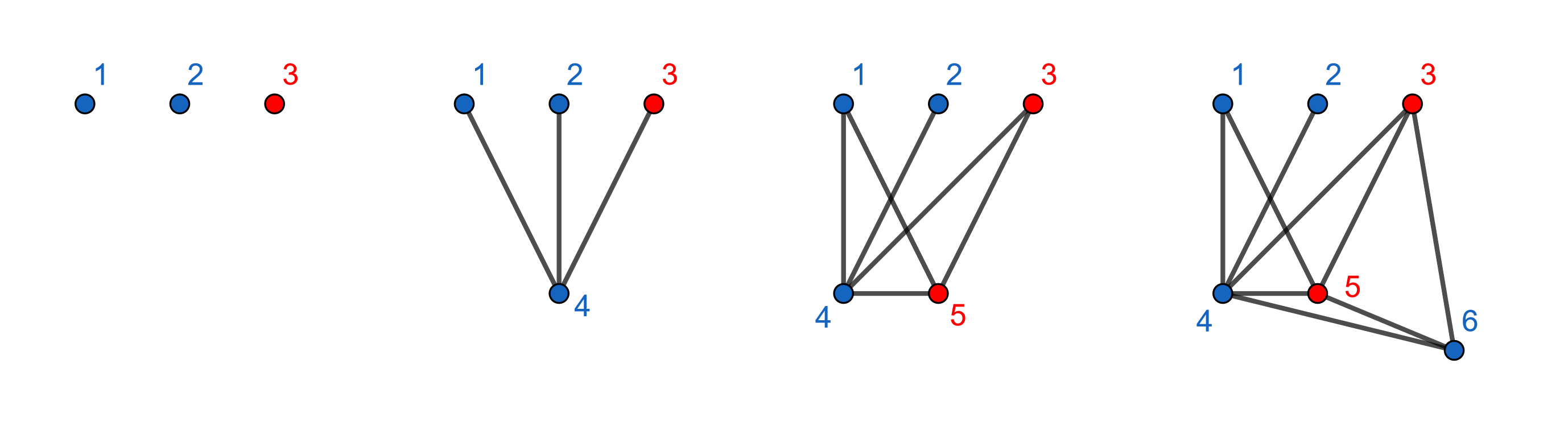}
\end{center}
\caption{A realisation of the process up to time 6, for $k=3$, starting with $R_3=1/3$.}
\end{figure}

If one is only interested in the evolution of the proportion of red and blue vertices (but not the structure of the graph), one may 
equivalently describe it by an urn model with multiple draws and random (nonlinear) replacement. 
The urn process is defined as follows.
The urn is initialized with an odd number $k$ of balls, a fraction $R_k$ being red and $B_k=1-R_k$ blue.
At each time $n \ge k+1,$

\begin{itemize}
\item $k$ balls are drawn from the urn, uniformly at random with replacement, and returned to the urn;
\item the 
  color of each drawn ball is flipped with probability $p$
  (i.e., a drawn ball that is red is observed as blue with probability $p$);
\item  a new ball is added to the urn, whose color is chosen as the
  majority of the $k$ observed colors.
\end{itemize}

In the root-bit estimation problem considered here, the statistician
has access to an unlabelled and undirected version of the graph at time $n$, along with the vertex colors.
The goal of the statistician is to estimate the colors assigned to the $k$ roots.
More precisely, based on the observed graph, one would like to guess 
the majority color at time $k$. 

This problem has been studied in depth by \cite{AdDeLuVe22} in the case when $k=1$, that is, when the 
produced graph is a uniform random recursive tree. 
Two types of methods for root-bit estimation were studied in \cite{AdDeLuVe22}.
One is based on first trying to localize the root of the tree--disregarding the vertex colors.
If one finds a vertex that is close to the root, one may use the color of that vertex as a guess
for the root color. Such a vertex is the centroid of the tree. Indeed, it is shown in \cite{AdDeLuVe22} that
the color of the centroid is a nearly optimal estimate of the root color. In the same paper, 
the majority rule is also studied. This method disregards the structure of the tree and guesses the
root color by taking a majority vote among all vertices. It is shown that for small mutation probabilities 
the majority rule is also nearly optimal.

In the more general problem considered in this paper, one may also try to estimate the colors of the $k$ roots
by finding nearby vertices. However, this problem becomes significantly more challenging as the 
$k$-{\sc dag} does not have a natural centroid. The interested reader is referred to the recent paper
of Briend, Calvillo, and Lugosi \cite{BrCaLu23} on root finding in random $k$-{\sc dag}s.
Instead of pursuing this direction, we focus on the majority vote. 
More precisely, we are interested in characterizing the values of the mutation probability $p$
such that the asymptotic probability of error is strictly better than random guessing.

At time $n$, the majority vote, denoted by $b^{maj}_n$, is defined as follows:

\begin{equation*} 
b^{maj}_n \ = \ \begin{cases}
 & \text{``R'' (red) if } R_n>1/2 \\
 \\
 & \text{``B''(blue) if } R_n<1/2 \\
 \\
 &\text{Ber}(1/2)\text{ if } R_n=1/2 \text{ (a random coin flip)} ~.
\end{cases}
\end{equation*}
We define the probability of error by
\[
 R^{maj}(n,p)=\PROB\left\{ b^{maj}_n\neq b^{maj}_k \right\}~.
\]
Note that $b^{maj}_k$ depends on the initial vertex colors that are
assumed to be chosen arbitrarily and fixed. Hence, $R^{maj}(n,p)$ is also a function
of the initial proportion $R_k$ but to avoid heavy notation, we supress this dependence.

\subsection{Related results and our contribution}
\label{sec:results}

Our broadcasting model is an extension of the broadcasting on uniform
random recursive trees that was extensively studied in
\citet{AdDeLuVe22}. In this problem, $k=1$ and the only parameter is
$p$, the mutation probability. For the majority voting rule, they
prove the following:

 \begin{enumerate}[label=(\roman*)]
\item There exists a constant $c>0$ such that 
  $$ \limsup_{n\to\infty} R^{maj}(n,p)\leq cp~. $$
\item
  For all $p\in (0,1/2]$,
  \[
     \lim_{n\to \infty} R_n = \frac{1}{2} \quad \text{with probability
       one}~.
   \]
\item For $p\in [0,1/4)$
 $$ \limsup_{n\to\infty} R^{maj}(n,p)<\frac{1}{2} ~.$$
 \item For $p\in [1/4,1/2]$
 $$ \limsup_{n\to\infty} R^{maj}(n,p)=\frac{1}{2}~. $$
\end{enumerate}

In other words, even though the proportion of vertices
that have the same color as the root converges to $1/2$, for mutation
probabilities smaller than $1/4$, sufficient information is preserved
about the root color for the majority vote to work with a nontrivial probability.

We generalize these results to $k$-{\sc dags} and characterize the
values of $p$ for which majority voting outperforms random guessing.
In order to state the main result of the paper, we introduce some
notation.

For any odd positive integer $k$, let 
\begin{equation}\label{eq:defalpha}
\alpha_k:= \frac{1}{2^{k-2}}\sum_{i>k/2}^k \binom{k}{i}(i-k/2) =4\EXP\left[ \left(\text{Bin}(k,1/2)-\frac{k}{2}\right)_+ \right]~.
\end{equation}
For example,  $\alpha_1=1$, $\alpha_3=3/2$, and by a simple
application of the central limit theorem, for large $k$,
\begin{equation}\label{eq:AlphaAsymptotic}
\alpha_k\sim \sqrt{\frac{2k}{\pi}}~.
\end{equation}

In the statement of our main theorem, we assume,
without loss of generality, that initially red vertices are in 
majority, that is, $R_k > 1/2$.

\begin{tcolorbox}
\begin{theorem}\label{thm:MainThm}
Let $k$ be an odd positive integer and consider the broadcasting
process on a random $k$-{\sc dag} described above. Assume that initially $R_k > 1/2$.
\begin{enumerate}[label=(\roman*)]
\item If $p<\frac{1}{2}-\frac{1}{2\alpha_k}$, then there exist
  $\beta_1 \in (0,1/2)$ and $\beta_2=1-\beta_1$ (whose value only
  depends on $k$ but not on the initial color configuration) such that
  \[
    \PROB\left\{ R_n\to \beta_1\right\}+\PROB\left\{ R_n\to
      \beta_2\right\}=1
    \quad \text{and} \quad \PROB\left\{ R_n\to
      \beta_1\right\}<\PROB\left\{ R_n\to \beta_2\right\}~.
  \]
In particular, regardless of the initial value of $R_k$, 
$$ \limsup_{n\to\infty}R^{maj}(n,p)<\frac{1}{2}~. $$
\item If $\frac{1}{2}-\frac{1}{2\alpha_k}\leq p<\frac{1}{2}-\frac{1}{4\alpha_k}$, then $R_n\to 1/2$ a.s. and
$$ \limsup_{n\to\infty}R^{maj}(n,p)<\frac{1}{2}~. $$
\item If $\frac{1}{2}-\frac{1}{4\alpha_k}\leq p\le \frac{1}{2}$ then $R_n\to1/2$ a.s. and
$$ \lim_{n\to\infty}R^{maj}(n,p)=\frac{1}{2}~. $$
\end{enumerate}
\end{theorem}
\end{tcolorbox}

Theorem \ref{thm:MainThm} shows that for all $k\ge 3$, there are three regimes of the
value of the mutation probability. In the low-rate-of-mutation regime
the proportion of red balls almost surely converges to one of two
numbers, both different from $1/2$.
Moreover, the limiting proportion is positively correlated with the
initial value. In the intermediate phase, the vertex colors are
asymptotically balanced, but there is enough signal for the majority
vote to perform strictly better than random guessing. Finally, in the
high-rate-of-mutation regime, the majority vote is equivalent to a
coin toss, at least asymptotically.

Note that for $k=1$, $\alpha_1=1$, so
$1/2-1/(2\alpha_1)=0$, and therefore 
the low-rate-of-mutation regime does not exist. Of course, this is in
accordance with the results of \cite{AdDeLuVe22} cited above.

On the other hand, for $k=3$ the two thresholds are
$1/2-1/(2\alpha_3)= 1/6$ and
$1/2-1/(4\alpha_1) =1/3$, meaning that
from $k=3$ onward the three different regimes can be observed. For
large $k$, both threshold values are of the order $1/2-\Theta(1/\sqrt{k})$.

A closely related model has been studied by \citet{Makur_2020}. They study
different random {\sc dags}, where important parameters are the number of
vertices at distance $k$ from the root and the indegree of vertices. They
also suppose that the position of the root
vertex is known. Two rules of root bit estimation are
studied: a noisy majority rule and the NAND rule. \citet{Makur_2020}
show that if the number of
vertices of depth $k$ is $\Omega\left(\log(k)\right)$ then there is a
threshold on the mutation probability for which root bit estimation is
possible.

As a first step, we study the convergence of the proportion of red
balls. To this end, it suffices to study the generalized urn process
defined above.  We mention here that \citet{CRIMALDI2022270} study a somewhat related urn
process, though with linear replacement rules.


In order to avoid unessential complications caused by breaking
  ties, we only consider odd values of $k$.
  The same techniques allow one to analyze even values of $k$. In such
  cases, in the event of a tie among the $k$ observed colors, one may
  choose the color of the new vertex at random.

\section{Different regimes}\label{sec:DifRegime}

We start by studying the evolution of $R_n$. Let us denote by $c_n$
the color of the $n$-th vertex appearing in the graph.
After possible mutation, each edge
connecting vertex $n+1$ to an older vertex carries a signal. This
signal is red with probability

$$ f(R_n):=(1-p)R_n+p(1-R_n)=(1-2p)R_n+p~. $$
Because the $k$ parents are chosen independently and that the color is chosen by the majority,

\begin{equation}
\PROB\left\{  c_{n+1}=R  \right\} \ = \ \PROB\left\{ \text{Bin}\left(k,f(R_n)\right)\geq k/2 \right\}~,
\end{equation}
where, conditionally on $R_n$, $\text{Bin}\left(k,f(R_n)\right)$ is a binomial random variable. Moreover, we know that the number of red vertices evolves as $(n+1)R_{n+1}=nR_n+\IND(c_{n+1}=R)$, where $\IND$ is the indicator function. We rewrite this as

\begin{equation}\label{ratio evolution}
R_{n+1}=R_n+\frac{\IND(c_{n+1}=R)-R_n}{n+1}~.
\end{equation}
A key to understanding $R_n$ is then to study the random variable
$\IND(c_{n+1}=R)-R_n$. We define, for $t\in [0,1]$,
\begin{equation}\label{function g}
g(t) \ := \ \E\left[ \IND(c_{n+1}=R)-R_n | R_n=t  \right] \ = \PROB\left\{\text{Bin}(k,f(t))>k/2\right\}-t~. 
\end{equation}
The evolution of $R_n$ is entirely determined by the function $g$.
Observe first that for any $t\in[0,1]$, $f(1-t)=1-f(t)$. Also, since $k$ is odd,
 
 $$ \PROB\left\{\text{Bin}(k,1-f(t))>k/2\right\}=1-\PROB\left\{\text{Bin}(k,f(t))>k/2\right\}~,$$
 which implies that 
$$ g(1-t)=-g(t)~. $$
The extremal values of $g$ are
 
$$g(0)=\PROB\left\{\text{Bin}(k,p)>k/2\right\}>0~,$$
and
$$g(1)=\PROB\left\{\text{Bin}(k,1-p)>k/2\right\}-1<0~.$$
Since $g$ is continuous, the polynomial $g$ has at least one root. From the symmetry property we have $g(1/2)=-g(1-1/2)=-g(1/2)$, so $g(1/2)=0$. Moreover we obtain

$$ g'(1/2)=\frac{1-2p}{2^{k-2}}\sum_{i>k/2}^k \binom{k}{i}(i-k/2)-1~.$$
Recalling the definition of $\alpha_k$ from \eqref{eq:defalpha},
we have $g'(1/2)=(1-2p)\alpha_k-1$. Since $\alpha_k\geq1$, we conclude:

\begin{equation*} 
  g'\left( \frac{1}{2}\right) \  \begin{cases} 
\  <0  & \text{if } p> \frac{1}{2}-\frac{1}{2\alpha_k}~, \\
 \\
\  >0 & \text{if } p< \frac{1}{2}-\frac{1}{2\alpha_k}~.
\end{cases}
\end{equation*}
To understand the other potential zeros of $g$, let us study its
convexity.

\begin{lemma}
The function $g$ is strictly convex on $(0,1/2)$ and strictly concave on $(1/2,1)$.
\end{lemma}

\begin{proof}
We may use the elementary identities

\begin{equation}\label{BetaEmbedding}
\PROB\left\{ \text{Bin}(k,x)\geq \frac{k+1}{2} \right\}=\PROB\left\{ \text{Beta}\left(\frac{k+1}{2},\frac{k+1}{2}\right)<x \right\}~,
\end{equation}
where Beta$(a,b)$ is a beta$(a,b)$ random variable. Hence,

\begin{equation*}
g(t)=\int_0^{f(t)}\left(x(1-x)\right)^{\frac{k-1}{2}}\frac{\Gamma(k+1)}{\Gamma^2\left(\frac{k+1}{2}\right)}dx \ \ -t~,
\end{equation*}
and therefore

\begin{equation}\label{eq:DerivativeOfG}
g'(t)=(1-2p)\left(f(t)(1-f(t))\right)^{\frac{k-1}{2}}\frac{\Gamma(k+1)}{\Gamma^2\left(\frac{k+1}{2}\right)}-1~.
\end{equation}
Since $f(t)(1-f(t))=-(1-2p)t(t-1)+p(1-p)$ is increasing for
$t\in(0,1/2)$ and decreasing for $t\in(1/2,1)$, $g$ is strictly convex
on $(0,1/2)$ and strictly concave on $(1/2,1)$.

\end{proof}

In summary, if $ p> \frac{1}{2}-\frac{1}{2\alpha_k}$, then
$g'(1/2)<0$, and thus $g$ is monotonically decreasing on $[0,1]$ and
has only one zero in $[0,1]$. If $g'(1/2)=0$, then there is only one
zero (at $1/2$) and $g$ exhibits an inflection point at $1/2$. If
$ p< \frac{1}{2}-\frac{1}{2\alpha_k}$, then $g'(1/2)>0$ and thus $g$
has exactly one zero in $(0,1/2)$ and by symmetry, it also has one
zero on $(1/2,1)$. We denote these zeros by $\beta_1$ and $\beta_2$,
respectively.

Figure \ref{fig:g} shows two examples of the graph of the function $g$.

\begin{figure}[H]
\begin{center}
\includegraphics[scale=0.6]{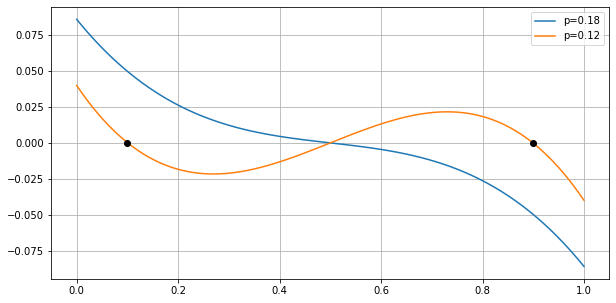}
\end{center}
\caption{$g$ as a function of $t\in[0,1]$, for $k=3$, with the choices
  $p=0.18>1/6$ and $p=0.12<1/6$.}
\label{fig:g}
\end{figure}

It is also interesting to know the position of $\beta_1$ (recall that
$\beta_2=1-\beta_1$). First, we note that for fixed $k$, if $p$ tends
to the threshold $1-1/(2\alpha_k)$, then $\beta_1$ tends to $1/2$.
In the following lemma we study the case when $p$ is far enough from the threshold, that is,
when $p\leq \frac{1}{2}-\frac{C}{2\alpha_k}$, for a sufficiently large constant
$C$.

\begin{lemma}
Let $p\leq \frac{1}{2}-\frac{C}{2\alpha_k}$ for $C\geq \sqrt{\frac{8\log(2)}{\pi}}$. Then
$$ \beta_1 \leq \exp\left(-\frac{k(1-2p)^2}{8}\right)~. $$
\end{lemma}

\begin{proof}
 $\beta_1$ is the smallest root of $g(t)$ and since $g(0)>0$, its
smallest root is smaller than the smallest root of any upper bound of
$g$. On the other hand, 
\[
  g(t) = \PROB\left\{ \text{Bin}(k,f(t))\geq \frac{k}{2} \right\} -t \leq
  \exp\left(
    -2k\left(\frac{1}{2}-f(t)\right)^2\right)=\exp\left(-2k(1-2p)^2\left(\frac{1}{2}-t\right)^2\right) -t~.
\]  
Thus, $\beta_1$ is at most the first zero of
$b(t):=\exp\left( c_1\left(\frac{1}{2}-t\right)^2\right)-t$, for
$c_1=2k(1-2p)^2$. Since $b(0)>0$, if for some $t^*$, $b(t^*)\leq0$
then the first zero of $b$ and therefore $\beta_1$ is at most
$t^*$. Taking $t^*=e^{-c_1/16}$, we have
\[
  b(t^*)\leq 0\iff \left(\frac{1}{2}-e^{-c_1/16}\right)^2\geq 1/16
  \iff c_1 \geq 32\log(2)~.
\]  
 From \eqref{eq:AlphaAsymptotic} and the expression of $c_1$, we have
 that by taking $C\geq \sqrt{\frac{8\log(2)}{\pi}}$, 
$$ 2k(1-2p)^2\geq 32\log(2)~. $$
This shows that for $p\leq \frac{1}{2}-\frac{C}{2\alpha_k}$, we have
$$ \beta_1 \leq \exp\left(-\frac{k(1-2p)^2}{8}\right)~. $$

\end{proof}

\section{Convergence of the proportion of red balls}
\label{sec:convergence}

In order to analyze the probability of error of the majority vote,
first we establish convergence properties of $R_n$.  The two possible
regimes of $g$ suggest that there are two distinct regimes of the
evolution of $R_n$. From \eqref{ratio evolution} we note that $R_n$
has a positive drift if $g(R_n)$ is positive, and a negative drift
otherwise. This suggests that in the high-rate-of-mutation regime,
$R_n$ converges to $1/2$ and in the low-rate-of-mutation regime it
converges to either $\beta_1$ or $\beta_2$. The following section
investigates this intuition, using Lemma 2.6 and Corollary 2.7 from
\citet{Pemantle_2007} about the convergence of reinforced random
processes. We state them here.

\begin{lemma}[\citet{Pemantle_2007}]\label{lem:PemantleNonVisit}
Let $\left\{ X_n; \ n\geq  0 \right\}$ be a stochastic process in $\R$ adapted to a filtration $\{ \mathcal{F}_n\}$. Suppose that $X_n$ satisfies 

$$ X_{n+1} - X_n =\frac{1}{n} \left( F(X_n)+\xi_{n+1}+E_n \right)~,  $$
where $F$ is a function on $\R$, $\E\left[ \xi_{n+1} \mid \mathcal{F}_n \right]=0$ and the remainder term $E_n$ goes to $0$ and satisfies $\sum_{n=1}^{\infty} n^{-1}|E_n| < \infty$ almost surely. Suppose that $F$ is bounded and that $\E\left[ \xi_{n+1}^2\mid \mathcal{F}_n \right] <K$ for some finite constant $K$. If for $a_0<x<b_0$, $F(x)\geq  \delta$ for some $\delta >0$, then for any $[a,b]\subset (a_0,b_0)$ the process $\{X_n\}$ visits $[a,b]$ finitely many times almost surely. The same result holds if $F(x)\leq -\delta$.
\end{lemma}

\begin{corollary}[\citet{Pemantle_2007}]\label{cor:PemantleConvergence}
If $F$ is continuous on $\R$, then $X_n$ converges almost surely to the zero set of $F$.
\end{corollary}

\subsection{The high-rate-of-mutation regime $\left( \frac{1}{2}-\frac{1}{2\alpha_k}\leq p\leq\frac{1}{2}\right)$}

Rewrite (\ref{ratio evolution}) as
\begin{equation*}
\begin{aligned}
R_{n+1}-R_n=& \frac{1}{n+1}\left( \PROB\left\{ \text{Bin}(k,f\left(R_n)\right)\geq \frac{k}{2} \right\}-R_n \right)\\
&+\frac{1}{n+1}\left(\IND(c_{n+1}=R)- \PROB\left\{ \text{Bin}(k,f\left(R_n)\right)\geq \frac{k}{2} \right\}\right)~.
\end{aligned}
\end{equation*} 
Since $g(R_n)=\PROB\left\{ \text{Bin}(k,f\left(R_n)\right)\geq k/2 \right\}-R_n$, we see that

\begin{equation}\label{pemantle ratio evolution}
R_{n+1}-R_n=\frac{g(R_n)+\xi_{n+1}}{n+1}~,
\end{equation}
where $ \xi_{n+1}= \IND(c_{n+1}=R)- \PROB\left\{ \text{Bin}(k,f\left(R_n)\right)\geq k/2 \right\}$. Because $g$ is continuous and $\E\left[\xi_{n+1}|R_n\right]=0$, our process satisfies all the requirements for Corollary \ref{cor:PemantleConvergence}. It states that $R_n$ converges almost surely to the set of zeros of $g$. In this regime, this implies that $R_n$ converges to $1/2$ almost surely.

\subsection{The low-rate-of-mutation regime $\left( 0<p<\frac{1}{2}-\frac{1}{2\alpha_k}\right)$}\label{sec:LowRateRegime}

In this regime, the requirements of Corollary \ref{cor:PemantleConvergence} are still met. So $R_n$ converges almost surely to the set of zeros of $g$, which is $\{\beta_1,1/2,\beta_2\}$. We first show that $R_n$ does not converge to $1/2$: $1/2$ seems to be an unstable equilibrium point, since the drift in the process has a tendency to pull $R_n$ away from $1/2$. We state Theorem 2.9 from \citet{Pemantle_2007} here:

\begin{theorem}[\citet{Pemantle_2007}]\label{thm:PemantleNonConvergence}

  Suppose $\{X_n\}$ satisfies the conditions of Lemma
  \ref{lem:PemantleNonVisit} and that for some $w \in (0, 1)$ and
  $\epsilon > 0$, $\text{sign}F(x)=\text{sign}(x-w)$ for all
  $x\in (w-\epsilon,w+\epsilon)$. For
  $\xi_{n+1}^+=\max(\xi_{n+1},0)$ and
  $\xi_{n+1}^-= \max(-\xi_{n+1},0)$, suppose that
  $\E[\xi_{n+1}^+\mid \mathcal{F}_n]$ and
  $\E[\xi_{n+1}^-\mid \mathcal{F}_n]$ are bounded above
  and below by positive numbers when
  $X_n \in (w - \epsilon, w + \epsilon)$. Then
  $\PROB\{ X_n \to w \}=0.$
\end{theorem}

\begin{corollary}
In the low-rate-of-mutation regime, almost surely the process $R_n$ does not converge to $\frac{1}{2}$.
\end{corollary}

\begin{proof}
Since the conditional distribution of $\xi_{n+1}$, given $R_n=1/2$ does not depend on $n$,
it is immediate that
$$ c<\E[\xi^+_{n+1}|R_n=1/2]<1~,$$
and
$$ c<\E[\xi^-_{n+1}|R_n=1/2]<1~,$$
for some $c>0$ that does not depend on $n$. Since $t\mapsto \E[\xi^{\pm}_{n+1}|R_n=t]$ is continuous and does not depend on $n$, there exists $\epsilon>0$ such that 
for all $t \in \left( 1/2-\epsilon, 1/2+\epsilon\right)$,
$$ \frac{c}{2}<\E[\xi^{\pm}_{n+1}|R_n=t]<2~. $$
Moreover, $g$ is negative on $(1/2-\epsilon,1/2)$ and positive on $(1/2,1/2+\epsilon)$. So, by Theorem \ref{thm:PemantleNonConvergence},
$$ \PROB\left\{ R_n\mapsto \frac{1}{2}\right\} =0~. $$
\end{proof}

\begin{corollary}
In the low-rate-of-mutation regime, the process $R_n$ converges almost surely, either to $\beta_1$ or to $\beta_2$, that is,
$$ \PROB\left\{ R_n\to \beta_1\right\}+\PROB\left\{ R_n\to \beta_2\right\}=1~. $$
\end{corollary}

\begin{proof}
It suffices to check that $R_n$ converges to $\beta_1$ or $\beta_2$ and does not oscillate between them. Between $1/2$ and $\beta_2$ the function $g$ is positive, so there exists $1/2<a_0<a_1<\beta_2$ and $\delta>0$ such that 
for all $t \in (a_0,a_1)$,
$  g(t)>\delta~.  $

Lemma \ref{lem:PemantleNonVisit} shows that $R_n$ visits any set
$[a,b]\subset(a_0,a_1)$ finitely often almost surely. Because the step
sizes of $R_n$ are of order $1/n$, if $R_n$ visits $[a,b]$ finitely
many times, it crosses it finitely many times. Indeed, for
$n$ large enough it cannot cross $[a,b]$ without visiting $[a,b]$.
Since $R_n$ converges almost surely to the set
$\{\beta_1,\beta_2\}$, but $R_n$ crosses the set $(a_0,a_1)$ finitely
many times, we see that $R_n$ converges almost surely either to
$\beta_1$ or $\beta_2$, as claimed.
\end{proof}

\section{Is majority voting better than random guessing?}
\label{sec:majority}

As a first step of understanding if majority voting is better than random guessing, we prove the following lemma. It gives an equivalent condition to the success of majority voting in terms of the first time the majority flips.

\begin{lemma}\label{lemma majority}
Let $T$ denote the random time at which the majority flips for the first time, that is,
$$ T=\min\left\{ n\in\N: \ b_n^{maj}	\neq b_k^{maj} \right\}~. $$
Then  $\limsup_{n\to\infty} R^{maj}(n,p)<1/2$ if and only if  
$\PROB\left\{ T=+\infty \right\}>0~.$
\end{lemma}

\begin{proof}
From the definition of $ R^{maj}(n,p)$, 
$$ \limsup_{n\to\infty} R^{maj}(n,p)=1-\liminf_{n\to\infty} \ \PROB\left\{b_n^{maj}=b_k^{maj}\right\}~.  $$
Fix a positive $\epsilon$. Since the sequence of events $\{\forall
i\in[n]: \ b_i^{maj}=b_k^{maj}\}$ is decreasing, and
$\{T=+\infty\}=\{\forall i >k; \  b_i^{maj}	=b_k^{maj}\}$, by continuity of measure we can choose $n$ such that
$$ \PROB\left\{ \forall i\in[n]: \ b_i^{maj}=b_k^{maj} \right\} \leq \PROB\left\{T=+\infty\right\}+\epsilon~. $$
For $N\geq n+1$, we have
\begin{equation}\label{event lemma}
\begin{aligned}
\PROB\left\{b_N^{maj}  =b_k^{maj} \right\}  = & \PROB\left\{ b_N^{maj}=b_k^{maj} \ \text{and} \ \forall i \in [n]: \ b_i^{maj}=b_k^{maj} \right\} \\
& + \PROB\left\{ b_N^{maj}=b_k^{maj} \ \text{and}\ \exists i \in [n]: \ b_i^{maj}\neq b_k^{maj} \right\}~.
\end{aligned}
\end{equation}
 The second term on the right-hand side decomposes as 
 \begin{eqnarray*}
\lefteqn{   
\PROB\left\{ b_N^{maj}=b_k^{maj} \ \text{and} \ \exists i \in [n]: \
   b_i^{maj}\neq b_k^{maj} \right\} } \\
   & = &
 \left( 1-\PROB\left\{ \forall i \in [n]: \ b_i^{maj}=b_k^{maj} \right\}\right)\PROB\left\{  b_N^{maj}=b_k^{maj} \ \Big{|} \ \exists i \in [n]:\ b_i^{maj}\neq b_k^{maj}   \right\} ~.
\end{eqnarray*}
From the definition of our process, if $R_i=1/2$, then,
conditionally on this event, the
distribution of $R_N$ for $N>i$ is symmetric. Therefore
\begin{equation}\label{eq:SymmetricAfterBalance}
 \PROB\left\{  b_N^{maj}=b_k^{maj} \ \Big{|} \ \exists i \in [n]: \ b_i^{maj}\neq b_k^{maj}   \right\}=\frac{1}{2}~.  
\end{equation}
Plugging this into \eqref{event lemma} yields
\begin{equation}\label{proba lemma}
\begin{aligned}
\PROB\left\{ b_N^{maj}=b_k^{maj} \right\}  = & \PROB\left\{ b_N^{maj}=b_k^{maj} \ \cap \ \forall i \in [n]: \ b_i^{maj}=b_k^{maj} \right\}  \\
&+ \frac{1}{2}\left( 1-\PROB\left\{ \forall i \in [n]: \ b_i^{maj}=b_k^{maj} \right\}\right) ~, 
\end{aligned}
\end{equation}
The first term of the right-hand side is bounded from below by $\PROB\left\{ T=+\infty \right\}$, which transforms \eqref{proba lemma} into
$$ \PROB\left\{ b_N^{maj}=b_k^{maj} \right\} \geq \frac{1}{2}+\PROB\left\{ T=+\infty\right\}-\frac{1}{2}\PROB\left\{ \forall i \in [n]: \ b_i^{maj}=b_k^{maj} \right\}~. $$
Taking the limit on $N$ and recalling the choice of $n$ gives 
$$ \liminf_{N\to\infty}\PROB\left\{ R_N>\frac{1}{2} \right\} \geq \frac{1}{2}+\frac{1}{2}\PROB\left\{ T=+\infty\right\}-\frac{\epsilon}{2}~.  $$
Since the above holds for any $\epsilon$, if $\PROB\left\{
  T=+\infty\right\}>0$ then $\liminf_{N\to\infty}\PROB\left\{
  b_N^{maj}=b_k^{maj} \right\}>1/2$. This proves the ``if''
direction of the statement.

On the other hand, from \eqref{proba lemma},
$$   \PROB\left\{ b_N^{maj}=b_k^{maj} \right\} \leq \PROB\left\{ \forall i \in [n]: \ b_i^{maj}=b_k^{maj} \right\} + \frac{1}{2}\left( 1-\PROB\left\{ \forall i \in [n]: \ b_i^{maj}=b_k^{maj} \right\}\right) ~.   $$
Taking the limit on $N$ and recalling the choice of $n$ yields
\begin{equation*}
\begin{aligned}
\liminf_{N\to\infty} \ \PROB\left\{ b_N^{maj}=b_k^{maj} \right\} & \leq \frac{1}{2}+   \frac{1}{2}\PROB\left\{ \forall i \in [n]: \ b_i^{maj}=b_k^{maj} \right\} \\
& \leq \frac{1}{2}+\frac{1}{2}\PROB\left\{ T=+\infty\right\}+\frac{\epsilon}{2} ~.
\end{aligned}
\end{equation*}
As this holds for any positive $\epsilon$, if $\liminf_{N\to\infty}
\PROB\left\{ b_N^{maj}=b_k^{maj} \right\}>1/2$, then $\PROB\left\{
  T=+\infty\right\}>0$. This concludes the proof.
\end{proof}

\begin{lemma} \label{Lem:SymmetricAfterBalance}
If 
$$\limsup_{n\to\infty} R^{maj}(n,p)\geq\frac{1}{2}~,$$
then
$$\lim_{n\to\infty} R^{maj}(n,p)=\frac{1}{2}~.$$
\end{lemma}

\begin{proof}
If $\limsup_{n\to\infty} R^{maj}(n,p)\geq\frac{1}{2}$ then Lemma
\ref{lemma majority} shows that $T$ is almost surely finite. 
But since 
$$ \PROB\left\{ b_n^{maj}\neq b_k^{maj} \mid T \le n \right\}=\frac{1}{2}~, $$
this implies
$$ \PROB\left\{ b_n^{maj}\neq b_k^{maj}, \ T\leq n \right\} = \frac{1}{2}\PROB\left\{ T\leq n \right\}~. $$
Moreover, since $T$ is finite almost surely,  $ \lim_{n\to \infty}
\PROB\left\{ T\leq n \right\}= 1$
 and by the continuity of measure,
$$ \lim_n \PROB\left\{ b_n^{maj}\neq b_k^{maj}, \ T\leq n \right\}= \PROB\left\{ b_n^{maj}\neq b_k^{maj}\right\}~. $$ 
This concludes the proof of the the lemma.
\end{proof}

\subsection{The low-rate-of-mutation regime $\left( 0<p<\frac{1}{2}-\frac{1}{2\alpha_k}\right)$}\label{subsec:lmregime}

As explained in Section \ref{sec:LowRateRegime}, if $p<\frac{1}{2}-\frac{1}{2\alpha_k}$, then  $R_n$ converges to either $\beta_1$ or $\beta_2$. Next we show that if $R_1>1/2$, then $R_n$ is more likely to converge to $\beta_2$ than to $\beta_1$. To do so, recall \eqref{ratio evolution} and write it as

$$ R_{n+1} = \frac{n}{n+1}R_n+\frac{1}{n+1} B_n( g(R_n)+R_n )~, $$
where the $B_n$ are independent Bernoulli random variables. We fix $\tau \in (1/2,\beta_2)$. From the analysis of $g$ we know that $g(\tau)>0$. Since $g(t)+t=\PROB\left\{ \text{Bin}(k,f(t))\geq k/2 \right\}$ and $f$ is increasing, for all $t\geq\tau$,
$$g(t)+t\geq g(\tau)+\tau~.$$ 
Fix a positive integer $N$ and introduce the mapping
$$t\mapsto h(t) :  \ \begin{cases}
 &  h(t)=1/2 \text{ if } t<\tau \\
 \\ & h(t)=g(\tau)+\tau \text{ otherwise} ~.
\end{cases}$$
Then define $D_k=1$. For $n\geq k$, let
$$D_{n+1}=\frac{n}{n+1}D_n+\frac{1}{n+1}B'_n\left(  h(D_n)\right)~,$$
where $B'_n$ are independent Bernoulli random variables. From the definition of the process $(D_n)$, on the event $\left\{ D_n\geq \tau, \ \forall n \geq 1 \right\}$
$$ nD_n \ge D_k+\text{Bin}(n-k,g(\tau)+\tau)~. $$
Hence, by the union bound and Hoeffding's inequality,
$$  \PROB\left\{ \exists i \geq N: \ D_i\leq \tau \ | \ \forall n \in [k,N]: \ D_n\geq \tau \right\} \leq \sum_{i\geq N} \PROB\left\{  \text{Bin}(i-k,g(\tau)+\tau)\leq i\tau \right\}\leq \frac{ 2e^{-(N-k)g(\tau)^2}}{1-e^{-2g(\tau)^2}}~. $$
Choosing $N$ such that the last term above is less than one yields
$$ \PROB\left\{ \forall i\geq N: \ D_i\geq \tau\ | \ \forall n\in [k,N]: \ D_n\geq \tau\right\}>0~. $$
Since
$$ \PROB\left\{ \forall i\geq k: \ D_i\geq \tau\right\} = \PROB\left\{ \forall i\in [k,N]: \ D_i\geq \tau\right\}\times\PROB\left\{ \forall i\geq N: \ D_i\geq \tau\ | \ \forall n\in [k,N]; \ D_n\geq \tau\right\}~,$$
we just proved that
\begin{equation}\label{eq:NotReachingTau}
 \PROB\left\{ \forall i\geq k: \ D_i\geq \tau\right\} >0~.
\end{equation}
Define the stopping time $T'=\min\left\{ n\geq k; \ D_n\leq\tau \right\}$. Since for all $t\geq\tau$, $g(t)+t\geq g(\tau)+\tau$, on the event $\left\{R_k\geq D_k\geq \tau\right\}$, there exists a coupling of the Bernoulli random variables $B$ and $B'$ such that
$$\forall n\in[k,T']: \ B_n\geq B'_n~,$$
and thus a coupling of the random variables $R_n$ and $D_n$ such that
$$ \forall n\in [k,T']: \ R_n\geq  D_n~. $$
From this coupling and \eqref{eq:NotReachingTau} we have
$$\PROB\left\{\forall n\geq k: \ R_n>\frac{1}{2}\right\} >0~,$$
which, thanks to Lemma \ref{lemma majority}, proves that in the regime $p < 1/2-1/(2\alpha_k)$,
$$  \limsup_{n\to\infty} R^{maj}(n,p)<\frac{1}{2}~, $$
proving the first statement of Theorem \ref{thm:MainThm}.

\subsection{The high-rate-of-mutation regime  $\left( \frac{1}{2}-\frac{1}{2\alpha_k}\leq p\leq\frac{1}{2}\right)$ }

In the range $p>1/2-1/(2\alpha_k)$ the proportion of red balls converges to $1/2$. It does not mean that majority voting can not be better than random guessing. Indeed, the proportion can converge to $1/2$ from above. This is this possibility that will now be investigated. 

\subsubsection{Extreme rate}

First, we examine the ``extreme'' case when the rate of mutation is near $1/2$, more precisely when $p>1/2-1/(4\alpha_k)$. Define the linear function $h$ by $h(t):=g'(1/2)\left( t-1/2 \right)$. Then

\begin{equation*}
g(t) \begin{cases}
& \geq h(t), \ \text{if }    t\in[0,1/2], \\
 & \leq h(t), \ \text{if } t\in[1/2,1]~.
\end{cases}•
\end{equation*}
In Figure \ref{fig:3} we plot $h$ and $g$.

\begin{figure}[H]
\begin{center}
\includegraphics[scale=0.6]{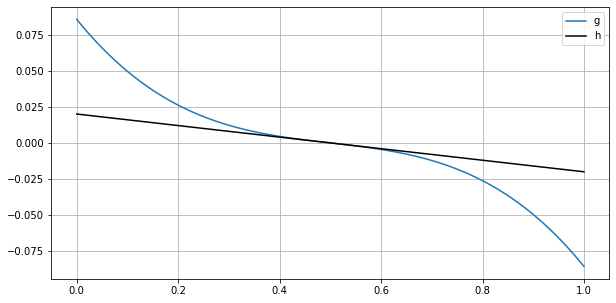}
\end{center}
\caption{A linear lower bound for $|g|$, $k=3$ and $p=0.18$.}
\label{fig:3}
\end{figure}
\noindent Let us define an auxiliary process $R^*_n$ by the stochastic recursion $R^*_k=1$ and for $n\geq k$

$$ R^*_{n+1}=R^*_n+\frac{B_n\left( h(R^*_n)+R^*_n \right)-R^*_n}{n+1}~, $$
where $B_n\left( h(R^*_n)+R^*_n \right)$ is a Bernoulli random variable with parameter $h(R^*_n)+R^*_n$, conditionally independent of $R^*_n$. In particular,

$$ \E\left[ B_n\left( h(R^*_n)+R^*_n \right)-R^*_n|R^*_n=t \right]=h(t)~.$$
Since the value of $g$ (for $(R_n)$) and $h$ (for $(R^*_n)$) represents a drift in the processes $R_n$ and $R^*_n$ we expect that the process $(R^*_n)$ is further away from $1/2$. Indeed, we may introduce a coupling as follows. Define the stopping time $T^*$ as the first time $R^*$ reaches $1/2$:

$$ T^*:=\min\left\{ n\geq k: \ R^*_n\leq \frac{1}{2} \right\}~. $$
Since for the times $n\in[k,T^*]$, $h(R^*_n)\geq g(R_n)$, we may use a similar coupling argument as in Section \ref{subsec:lmregime}. Thus, there is a coupling of $R^*$ and $R$ such that

$$ \forall n \in [k,T^*]; \ R_n\leq R^*_n~. $$
From this coupling, for $T$ defined in Lemma \ref{lemma majority} we have

\begin{equation}\label{majority not working}
\PROB\left\{ T=+\infty\right\} \leq \PROB\left\{ T^*=+\infty \right\}~.
\end{equation}
Observe that in the case of $k=1$, $g$ is linear and the two processes $R_n$ and $R^*_n$ coincide. The linear case was analyzed in \citet{AdDeLuVe22} and we may use their results to understand the behavior of $R_n^*$. Indeed, the process defined in \citet{AdDeLuVe22} is the same as $R^*$ if one sets the flip probability of \citet{AdDeLuVe22} equal to $-g'(1/2)/2$  and starts at time $k$. They prove that if $p\geq 1/4$, then, for the process starting at time $1$, majority voting has an error probability of $1/2+o(1/2)$.  Lemma \ref{lemma majority} implies that this process reaches $1/2$ in finite time almost surely. So even conditioned on its value being $1$ at time $k$ it will reach $1/2$ in finite time almost surely. This proves that even for $R_n^*$ starting at time $k$ its error probability is $1/2+o(1)$.   According to Lemma \ref{lemma majority} this implies that for this range of $p$, $\PROB\left\{ T^*=+\infty \right\}=0$. Hence, using Lemma \ref{lemma majority} and \eqref{majority not working}, shows that if $g'(1/2) \leq -\frac{1}{2},$ then

$$  \limsup_{n\to\infty} R^{maj}(n,p)=\frac{1}{2}~.  $$
Lemma \ref{Lem:SymmetricAfterBalance} shows that $\lim_{n\to\infty} R^{maj}(n,p)=1/2$. Because $g'(1/2)=(1-2p)\alpha_k-1$, we just proved that if $p\geq 1/2-1/4\alpha_k$, then

$$ \lim_{n\to\infty} R^{maj}(n,p)=\frac{1}{2}~, $$
completing the proof of the third statement of Theorem \ref{thm:MainThm}.

\subsubsection{Intermediate rate}
\label{sec:intermediate}

It remains to study the ``intermediate'' case $p\in[1/2-1/(2\alpha_k),1/2-1/(4\alpha_k))$. To this end, we may couple $R_n$ to a process for which majority voting outperforms random guessing. Let us fix $p\in[1/2-1/(2\alpha_k),1/2-1/(4\alpha_k))$, which implies that $g'(1/2)/2>-1/4$. Then choose $q=-g'(1/2)/2+\epsilon$ with $\epsilon>0$ small enough so that $q<1/4$ and $g(0)>h(0)$. We define the linear function $h(t):=-2q(t-1/2)$, and as illustrated in Figure \ref{fig:4}, we denote by $a$ and $b$ the intersection points between $h$ and $g$ (apart from $1/2$). More precisely $a$ and $b$ are defined as the the roots of $g-h$ distinct from $0$. Since $g-h$ is strictly convex on $(0,1/2)$ and $(g-h)(0)>0$, $(g-h)'(1/2)<0$, $a$ and $b$ are well defined and sit respectively in $(0,1/2)$ and $(1/2,1)$.

\begin{figure}[H]
\begin{center}
\includegraphics[scale=0.6]{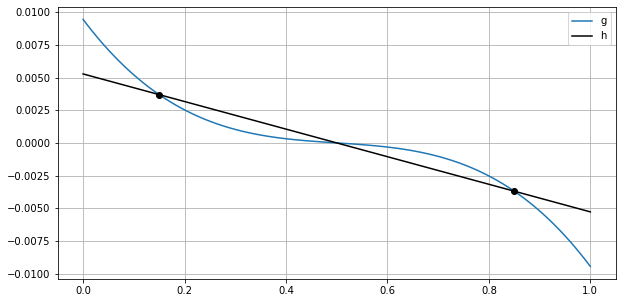}
\end{center}
\caption{Comparison of $h$ and $g$, for $k=3$ and $p=0.34$ (rescaled for clarity).}
\label{fig:4}
\end{figure}
\noindent We define $R^*_n$ similarly as in the previous section but now with $h(t)=-2q(t-1/2)$, that is $R^*_k=1$ and

$$ R^*_{n+1}=R^*_n+\frac{B_n\left( h(R^*_n)+R^*_n \right)-R^*_n}{n+1}~, $$
where the $B_n$ are conditionally independent Bernoulli random variables. In particular,

$$ \E\left[B_n\left( h(R^*_n)+R^*_n \right)-R^*_n|R^*_n=t \right]=-2q\left( t-\frac{1}{2} \right)~.$$
Just as in the previous section, we may use the analysis of
\citet{AdDeLuVe22} for the case $k=1$ with mutation probability of
$q$. \citet{AdDeLuVe22} state that  for the process $R^*_n$
started at time $1$ and for  $q<1/4$ majority voting is better than
random guessing.
As it was pointed out to us by a referee, the proof in
  \citet{AdDeLuVe22} that uses general  limit theorems for Pólya urns with randomized
replacements due to Janson \cite{10.1214/19-ECP226} is incorrect. In
the Appendix we give a self-contained proof of this statement.

To use the result for trees, we need to make sure that
  it holds when $R^*_n$ is defined as above, started at time $k$. Let
  $R'^*_n $ be the process started at time $1$, for which the majority
  voting is known to outperform random guessing. Let $T'^*$ and $T^*$
  be the first time indices at which $R'^*_n$ and $R^*_n$ reach $1/2$,
  respectively. Finally, let $(U_n)_{n\in \N}$ be a collection of independent uniform random variables. For $n\in[k,T'^*)$, we couple $R'^*_n$ and $R^*_n$ as follows
$$ R'^*_{n+1}=R'^*_n+\frac{\IND\big(U_n\leq h(R'^*_n)+R'^*_n \big)-R'^*_n}{n+1}~, $$
and
$$ R^*_{n+1}=R^*_n+\frac{\IND\big(U_n\leq h(R^*_n)+R^*_n \big)-R^*_n}{n+1}~.$$
With this coupling, a recursion proves that for all $n\in[k,T'^*)$, $R^*_n\geq R'^*_n$.   Because majority voting is known to outperform random guessing for $R'^*$, Lemma \ref{lemma majority} proves that $\PROB\left\{T'^*=+\infty \right\}>0$.  The coupling directly implies that $\PROB\left\{T^*=+\infty \right\}>0$. So majority voting outperforms random guessing for the process $R^*$.  Thus, from Lemma \ref{lemma majority} it follows that 

$$ \PROB\left\{ \forall n \geq k: \ R^*_n>\frac{1}{2} \right\}>0~. $$
Now, from Lemma \ref{lem:PemantleNonVisit} we deduce that both processes $R_n$ and $R^*_n$ converge almost surely to $1/2$ and exceed $b$ only finitely many times. Thus, there exists an almost surely finite random time  $T'$ such that and $\forall n \geq T';$ $R_n\leq b$ and $R^*_n\leq b$. We use similar coupling arguments as in Section \ref{subsec:lmregime}. So, on the event that $R^*$ does not reach $1/2$ we can couple $R_n$ and $R^*_n$ from $T'$ onwards such that $R_n\geq R^*_n$. This proves that

$$ \PROB\left\{ \forall n \geq T': \ R_n>\frac{1}{2} \ | \ T' \right\} >0~. $$
Using that $T'$ is finite almost surely and Lemma \ref{lemma majority} we conclude that majority voting is better than random guessing in this regime. More precisely, if $1/2-1/2\alpha_k\leq p <1/2-1/4\alpha_k$, then

$$\limsup_{n\to\infty} R^{maj}(n,p)<\frac{1}{2}~.$$
This completes the proof of Theorem \ref{thm:MainThm}.

\section{A general lower bound}
\label{sec:lower}

In this final section we derive a lower bound for the probability of error that holds for all mutation probabilities. In particular we show the following.

\begin{proposition}
Let $k$ be a positive odd integer and let $k/2<\ell<k$. Assume that
initially there are $\ell$ red vertices, that is $R_k=\ell/k$.
Letting
\[
  h_k := \PROB\left\{ \text{Beta}\left(
      \frac{k+1}{2},\frac{k+1}{2}\right)\geq 1-\frac{1}{k} \right\}~,
\]  
the probability of error of the majority rule satisfies
\[
  \inf_{0\leq p \leq  1\atop \ n\geq 2\ell} \PROB\left\{ b_n^{maj}\neq b_k^{maj} \right\} \geq \frac{1}{2}h_k^{2\ell-k}~. 
\]
\end{proposition}  

\begin{proof}
The proposition follows by simply considering the event that
the first $2\ell-k$ new vertices are all blue. In that case, at time
$2\ell$ the number of red and blue vertices are equal.
We may write, for any $n\ge 2\ell$,
\begin{align*}
\PROB\left\{ b_n^{maj}\neq b_k^{maj} \right\} & \geq \PROB\left\{ b_n^{maj}\neq b_k^{maj} \ | \ c_{k+1}=\cdots= c_{2\ell }=B \right\}\times \PROB\left\{  c_{k+1}=\cdots= c_{2\ell }=B \right\}  ~.
\end{align*}
From the symmetry of our model, $\PROB\left\{ b_n^{maj}\neq b_k^{maj}  \ | \  c_{k+1}=\cdots = c_{2\ell }=B\right\}=1/2$. Thus
$$ \PROB\left\{ b_n^{maj}\neq b_k^{maj} \right\} \geq \frac{\PROB\left\{  c_{k+1}=\cdots c_{2\ell }=B \right\}}{2}~. $$
To estimate the probability on the right-hand side, we use
\eqref{BetaEmbedding}, which implies
\[
  \PROB\left\{ c_{i}=B
  \right\}=\int_{f(R_i)}^1\left(x(1-x)\right)^{\frac{k-1}{2}}\frac{\Gamma(k+1)}{\Gamma^2\left(\frac{k+1}{2}\right)}dx
  ~.
\]  
If $R_k=\ell/k$ and $c_{k+1}=\cdots = c_{i-1}=B$, where $k<i\leq 2k$, then $R_{i-1}=\ell/i$. Since $0\leq p \leq 1/2$,
\[
  f(R_{i-1})=(1-2p)\frac{\ell}{i}+p\leq \max\left( \frac{1}{2},
    \frac{\ell}{i} \right) \leq \frac{k-1}{k}=1-\frac{1}{k}~.
\]  
Therefore,

$$ \min_{k<i\leq 2\ell} \PROB\left\{ c_i=B \ \mid \ c_{k+1}=\cdots = c_{i-1}=B \right\}\geq h_k~,$$
as claimed.
\end{proof}

\section{Concluding remarks}

In this paper we study the majority rule for guessing the initial bit
values at the roots of a random recursive $k$-{\sc dag} in a
broadcasting model. The main
result of the paper characterizes the values of the mutation probability 
for which the majority rule performs strictly better than random
guessing. Even in this exact model, many interesting questions remain
open. For example, we do not have sharp bounds for the probability of
error.
It would also be interesting to study other, more sophisticated,
classification rules that take the structure of the observed 
$k$-{\sc dag} into account. In particular, the optimal probability of
error (as a function of $k$ and the mutation probability $p$) is
far from being well understood. For an initial study of localizing the
root vertices, we refer the interested reader to \citet{BrCaLu23}.

 A natural extension of the model is obtained by considering
  $q>2$ colors. In this model, one aims at guessing the most common
  color of the initial configuration. In order to extend our results,
  instead of a single number, one needs to consider a vector of dimension
  $q-1$  to track the proportion of each color. For example, one may
  consider the following rule to assign a color to a new
  vertex. At each step, among the $k$ observed colors, pick the most
  common (break ties uniformly
  at random). However, the analysis becomes most complex since
  instead of comparing one random
  variable to $1/2$, one needs to
  compare a random variable to $q-2$ others to determine which one is
  the most common. If one manages to write this recursion in a
  tractable manner, we believe that a similar approach as the one of
  this paper may be used to understand the evolution of the proportion
  of each color. Depending on the convergence regime, an important
  part of our proof relies on the comparison to the tree case, that
  is, the case $k=1$. In a tree, one way to study the multi-color problem
  is to group $q-1$ of the colors together. By doing so, the
  multi-color problem is simplified to a two-color problem in the tree
  with a non-symmetric flip probability.
  However, the details may be nontrivial and are left to future research.

\subsection*{Acknowledgements}
  We thank the referees for their interesting comments. 
  We are especially grateful to a referee for pointing out a mistake in
  \cite{AdDeLuVe22} which is fixed in the Appendix below.
 
\section{Appendix}
\subsection{In trees majority is better than random guessing for $p < 1/4$}

In Section 2.3 of \cite{AdDeLuVe22} it is stated that the majority
vote is asymptotically better than random guessing when $k=1$ (i.e.,
the graph is a uniform random recursive tree) and the mutation
probability $p$ is less than $1/4$. This result is used in Section
\ref{sec:intermediate} above.

In \cite{AdDeLuVe22} it is claimed that this follows simply from
a general limit theorem for Pólya urns with randomized
replacements due to Janson \cite{10.1214/19-ECP226}. However, the
claimed symmetry in Janson's limit distribution was not checked in
\cite{AdDeLuVe22}.
In order to remedy this, in the next Proposition we give a self-contained proof of the
statement.

Suppose, without loss of generality, that the root vertex is red, that
is, $R_1=1$. Define the difference between the number of red and blue
balls at time $n$ by $\Delta_n= n(R_n-B_n)$.

\begin{proposition}
If $p<1/4$, then
  \[
    \liminf_{n\to \infty} \PROB\left\{ \Delta_n > 0 \right\} > \frac{1}{2}~.
 \]
\end{proposition}

\begin{proof}
We use the representation defined in Section 2.1 of
\cite{AdDeLuVe22} for the difference $\Delta_n$, which we recall now.

The URRT is generated in the standard way, without 
attached colors, with $0$ being the root vertex, and for $i\in
\{1,\ldots,n\}$,
$p_i\in \{0,\dots , i-1\}$ is the uniform random index of the parent
of vertex $i$.
Coloring the vertices may be equivalently done as follows:
\begin{itemize}
\item let $M_1,M_2,\ldots,M_n \in \{0,1\}$ be independent Bernoulli$(2p)$ random variables. When $M_i=1$,
  vertex $i$ is \emph{marked}. Then there is an independent coin flip $\xi_i$ that takes values uniformly at random in  $\{-1,1\}$ and determines if a marked node takes the same color as its parent or it flips.
\item when $M_i=0$, vertex $i$ is \emph{not marked}. These nodes
  have the same color as their parent.
\end{itemize}
The root and marked nodes become roots of subtrees that are disjoint
and partition the uniform recursive tree into many pieces. Each of the
subtrees consists of nodes of the same color, and the roots have the color of their original parent if $\xi =1$ and different otherwise (if $\xi_i =-1$). 
More precisely, if $B_i\in \{-1,1\}$ is the color of
vertex $i$ (with $+1$ interpreted as ``red'' and $-1$ as ``blue''), then 

\[ B_i = \left\{ \begin{array}{ll}
         B_{p_i} & \mbox{if $M_i=0$ (no marking) or if $M_i=1,\ \xi_i=+1$ (no flipping)};\\
        -B_{p_i} & \mbox{if $M_i=1,\ \xi_i=-1$}.\end{array} \right. \]   Let $N_i$ denote the size of the maximal subtree rooted at vertex $i$
such that all its vertices apart from $i$ are unmarked (and therefore monochromatic).


With this notation, $\Delta_n$ may be written as
$\Delta_n = N_0+W_n$, where
\[
   W_n= \sum_{i=1}^n N_i B_{p_i} \xi_i M_i~.
 \]
Since the Rademacher random variables $\xi_i$ are independent of all
other random variables, $W_n$ has a symmetric distribution about $0$.
In particular, by conditioning on all other random variables, we get
that
\[
   \PROB\left\{ \Delta_n \le 0\right\}= \frac{1}{2}\PROB\left\{ N_0 \le |W_n| \right\}~.
 \]
Hence, it suffices to show that $\limsup_{n\to \infty} \PROB\left\{
  N_0 \le |W_n| \right\} < 1$.
To this end,
for a positive integer $k$,  let $\mathcal{E}_k$ be the event that
the first $k$ vertices are unmarked, that is, $M_1=\cdots = M_k=0$.
Clearly, $\PROB\{\mathcal{E}_k\}= (1-2p)^k$.
Then
\begin{eqnarray*}
  \PROB\left\{     |W_n| \ge  N_0 \right\}
 & \le & \PROB\{\mathcal{E}_k^c\} +   \PROB\left\{     |W_n| \ge
    N_0,  \mathcal{E}_k\right\} \\
 & = & 1- \PROB\{\mathcal{E}_k\} + \PROB\left\{     \left|\sum_{i=k+1}^n N_i B_{p_i} \xi_i M_i\right| >
       N_0,  \mathcal{E}_k\right\}\\
 & = & 1- \PROB\{\mathcal{E}_k\} \left(1- \PROB\left\{     \left|\sum_{i=k+1}^n N_i B_{p_i} \xi_i M_i\right| >
       N_0\right\} \right)~,
\end{eqnarray*}
where we used the fact that the events $\mathcal{E}_k$
and $\left|\sum_{i=k+1}^n N_i B_{p_i} \xi_i M_i\right| > N_0$
are independent.
Thus, it suffices to prove that there exists an integer $k>0$ such that
$\limsup_{n\to\infty} \PROB\left\{     \left|\sum_{i=k+1}^n N_i B_{p_i} \xi_i M_i\right| >
       N_0\right\} <1$.
To this end, we may write
\begin{equation}
\label{eq:bound}
  \PROB\left\{     \left|\sum_{i=k+1}^n N_i B_{p_i} \xi_i M_i\right| >
       N_0\right\}
 \le  \PROB\left\{     \left|\sum_{i=k+1}^n N_i B_{p_i} \xi_i M_i\right| >
      \frac{\EXP N_0}{2}\right\}
       + \PROB\left\{     N_0 <  \frac{\EXP N_0}{2} \right\}~.
     \end{equation}
We show that the second term on the right-hand side is bounded away
from one, while the first term can be made arbitrarily small by
choosing $k$ sufficiently large.
In order to bound the second term on the right-hand side of
\eqref{eq:bound}, we use
the fact that by \cite[Lemmas 4 and 6]{AdDeLuVe22},
$\EXP N_0 \ge e^{-1}(n+1)^{1-2p}$ and
$\var(N_0) \le c(p) (n+1)^{2-4p} + O(n\log n)$ for a 
constant $c(p)>0$ depending on $p$ only. Hence, using the
Chebyshev-Cantelli inequality,
\[
  \PROB\left\{     N_0 \ge \frac{\EXP N_0}{2} \right\}
  \ge \frac{(\EXP N_0)^2}{(\EXP N_0)^2+4 \var(N_0)}
  \ge \frac{e^{-2}}{e^{-2}+c(p) +o_n(1)}~,
\]
which is clearly bounded away from $0$.

Using once again the bound
$\EXP N_0 \ge e^{-1}(n+1)^{1-2p}$,
the first probability on the right-hand side on \eqref{eq:bound} may be upper bounded, using
Markov's inequality, by
\[
   \PROB\left\{     \left|\sum_{i=k+1}^n N_i B_{p_i} \xi_i M_i\right|\ge
     \frac{ (n+1)^{1-2p}}{2e}  \right\}
   \le \frac{\left( \EXP \left(\sum_{i=k+1}^n N_i B_{p_i} \xi_i
         M_i\right)^2\right)^{1/2}}{ (n+1)^{1-2p}/2e}
\]
Since the  $\xi_i$ are independent of all
other random variables, 
\[
   \EXP \left(\sum_{i=k+1}^n N_i B_{p_i} \xi_i
     M_i\right)^2
   = \sum_{i=k+1}^n\EXP N_i^2 \le
   \sum_{i=k+1}^n \left( \left(\frac{n+1}{i+1}\right)^{2-4p} e^4(4+e)
     + e \right)~,
 \]
where the upper bound for $\EXP N_i^2$ follows from \cite[Lemma
6]{AdDeLuVe22}. Thus,
\begin{eqnarray*}
   \PROB\left\{     \left|\sum_{i=k+1}^n N_i B_{p_i} \xi_i M_i\right|\ge
     \frac{(n+1)^{1-2p}}{2e}  \right\}
  & \le & \left(2e^3\sqrt{4+e}\right) \left(\sum_{i=k+1}^n
          \frac{1}{(i+1)^2}\right)^{1/2} + 2e^{3/2} n^{-1/2+2p}  \\
  & \le &
          \frac{2e^3\sqrt{4+e}}{\sqrt{k}} + o_n(1)~.
\end{eqnarray*}
Thus, by choosing $k$ sufficiently large, we clearly have
$\limsup_{n\to\infty} \PROB\left\{     \left|\sum_{i=k+1}^n N_i B_{p_i} \xi_i M_i\right| >
       N_0\right\} <1$ as desired.
\end{proof}


\end{document}